\documentclass{article}

\usepackage{microtype}
\usepackage{graphicx}
\usepackage{subfigure}
\usepackage{booktabs} 
\usepackage{appendix}
\usepackage{pifont}
\usepackage{color}

\usepackage{hyperref}
\usepackage{chngcntr}

\usepackage[accepted]{icml2023}

\usepackage{listings}
\usepackage{enumitem}
\usepackage{wrapfig}
\usepackage{lipsum}

\usepackage{amsmath}
\usepackage{amssymb}
\usepackage{mathtools}
\usepackage{amsthm}
\usepackage{multirow}
\usepackage{makecell}
\usepackage{bigstrut}
\usepackage[capitalize,noabbrev]{cleveref}

\theoremstyle{plain}

\newtheorem{proposition}{Proposition}

\theoremstyle{definition}

\theoremstyle{remark}

\usepackage[textsize=tiny]{todonotes}

\icmltitlerunning{Model-Aware Contrastive Learning: Towards Escaping the Dilemmas}

\begin{document}

\twocolumn[
\icmltitle{Model-Aware Contrastive Learning: Towards Escaping the Dilemmas}




\begin{icmlauthorlist}
\icmlauthor{Zizheng Huang}{nju}
\icmlauthor{Haoxing Chen}{nju,ant}
\icmlauthor{Ziqi Wen}{telcom}
\icmlauthor{Chao Zhang}{nju}
\icmlauthor{Huaxiong Li}{nju}
\icmlauthor{Bo Wang}{nju}
\icmlauthor{Chunlin Chen}{nju}

\end{icmlauthorlist}

\icmlaffiliation{nju}{Nanjing University}
\icmlaffiliation{telcom}{China Telecom}
\icmlaffiliation{ant}{Ant Group}

\icmlcorrespondingauthor{Huaxiong Li}{huaxiongli@nju.edu.cn}
\icmlcorrespondingauthor{Haoxing Chen}{hx.chen@hotmail.com}
\icmlcorrespondingauthor{Zizheng Huang}{zizhenghuang@smail.nju.edu.cn}

\icmlkeywords{Representation Learning, Contrastive Learning}

\vskip 0.3in
]



\printAffiliationsAndNotice{}  

\begin{abstract}
Contrastive learning (CL) continuously achieves significant breakthroughs across multiple domains. However, the most common InfoNCE-based methods suffer from some dilemmas, such as \textit{uniformity-tolerance dilemma} (UTD) and \textit{gradient reduction}, both of which are related to a $\mathcal{P}_{ij}$ term. It has been identified that UTD can lead to unexpected performance degradation. We argue that the fixity of temperature is to blame for UTD. To tackle this challenge, we enrich the CL loss family by presenting a Model-Aware Contrastive Learning (MACL) strategy, whose temperature is adaptive to the magnitude of alignment that reflects the basic confidence of the instance discrimination task, then enables CL loss to adjust the penalty strength for hard negatives adaptively. Regarding another dilemma, the gradient reduction issue, we derive the limits of an involved gradient scaling factor, which allows us to explain from a unified perspective why some recent approaches are effective with fewer negative samples, and summarily present a gradient reweighting to escape this dilemma. Extensive remarkable empirical results in vision, sentence, and graph modality validate our approach's general improvement for representation learning and downstream tasks.
\end{abstract}

\begin{figure}[t]
	\center
	\includegraphics[width=7cm]{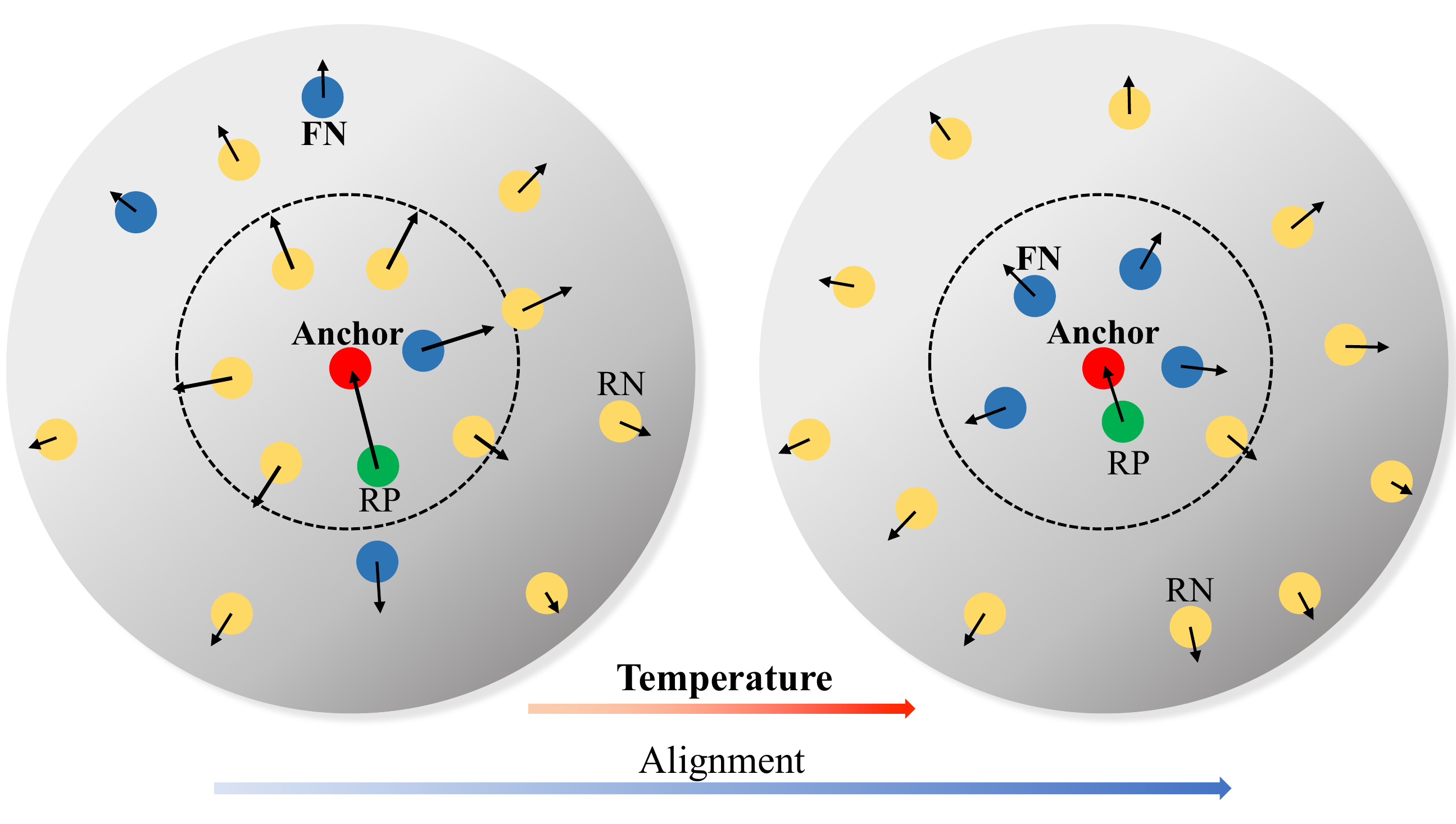}
	\caption{\textbf{Illustration of model-aware temperature strategy.} Points in red, green, yellow, and blue on the hypersphere denote \textcolor[RGB]{255,0,0}{anchor}, the \textcolor[RGB]{0,176,80}{real positive sample (RP)}, \textcolor[RGB]{255,217,102}{real negative samples (RN)}, and \textcolor[RGB]{46,117,182}{false negatives (FNs)}, respectively. Since alignment magnitude can indicate discrimination confidence of the CL model, then the alignment-adaptive temperature dynamically controls penalty strength (arrow length) to negative samples to balance uniformity and tolerance for samples.}
	\label{frame}
\end{figure}

\section{Introduction}
Modern representation learning has been greatly facilitated by deep neural networks \cite{bengio2013representation,dosovitskiy2020image,he2016deep,vaswani2017attention}. Self-supervised learning (SSL) is one of the most popular paradigms in the unsupervised scenario, which can learn transferable representations without depending on manual labeling \cite{gidaris2018unsupervised,he2022masked,grill2020bootstrap}. Especially, SSL methods based on contrastive loss have highly boosted CV, NLP, graph, and multi-modal tasks \cite{chen2020big,he2020momentum,you2021graph,gao2021simcse,radford2021learning}. These contrastive learning (CL) frameworks generally map raw data onto a hypersphere embedding space, whose embedding similarity can reflect the semantic relationship \cite{Wu_2018_CVPR,he2020momentum}. Among diverse contrastive losses, InfoNCE \cite{van2018representation,tian2020contrastive} is widely adopted in various CL algorithms \cite{chen2020simple,chen2021empirical,dwibedi2021little}, which attempts to attract positive samples to the anchor while pushing all the negative samples away.

InfoNCE loss is essential to the success of CL \cite{tian2022deep,wang2020understanding} but still troubled by several dilemmas. An interesting hardness-aware property has been pointed out, which enables CL automatically concentrate on hard negative samples (HNs, those having high similarities with the anchor) \cite{wang2021understanding,tian2022deep}. Particularly, the temperature parameter $\tau$ determines the weight distribution on negatives. But this also causes a Uniformity-Tolerance Dilemma (UTD) that plagues CL performance \cite{wang2021understanding}. Specifically, as for the common instance discrimination task in CL, models are trained by maximizing the similarities of the anchor with its augmentations and minimizing that of all the other different instances \cite{Wu_2018_CVPR,tian2020makes}. Such a strategy neglects the underlying semantic relationships, which can be explicitly subscribed by labels when in the supervised scenario. Those HNs might contain false negative samples (FNs) in this context. Owing to the hardness-aware property, a smaller $\tau$ is conducive to the uniformity of the embedding space \cite{wang2020understanding}, but goes against FNs due to excessive penalties on HNs. On the contrary, larger temperature parameters are beneficial for exploring underlying semantic correlations, while detrimental for learning separable informative features.

This work mainly focuses on two dilemmas in CL, both of which are related to a $\mathcal{P}_{ij}$ term. (1) \textit{The uniformity-tolerance dilemma}, which is still an open problem in contrastive learning. We argue that a training-adaptive temperature is key to alleviating UTD. In the learning phase, alignment of positive paris \cite{wang2020understanding} exactly can reflect the prior expectation of the instance discrimination task but also needs no extra computations in InfoNCE. Specifically, its alignment is underperforming for a poorly trained CL model. In this case, a smaller temperature parameter does help to improve the uniformity of the hypersphere embedding space \cite{wang2020understanding}. In contrast, the well-trained one is much better in terms of alignment, for which a larger temperature contributes to the tolerance for latent semantic relationships. Thus, we propose a model-aware temperature strategy based on alignment to solve the UTD problem. This strategy is illustrated in Figure \ref{frame}. (2) \textit{The gradient reduction dilemma of InfoNCE}. We identify the importance of negative sample size $K$ and temperature $\tau$ for this gradient reduction problem. From a unified perspective, two propositions explain why some previous work \cite{yeh2022decoupled,zhang2022dual,chen2021simpler} are experimentally valid. As a result, we also provide a reweighting method for learning with small negative sizes. Owing to these explorations and Model-Aware Contrastive Learning (MACL) strategy, we reconstruct the contrastive loss to enable CL models to generate high-quality representations. Experiments and analysis on some benchmarks in different modalities demonstrate that the proposed MACL strategy does help improve the learned embeddings and escape dilemmas.

\section{Related Work}
Self-supervised learning has achieved significant success, which can provide semantically meaningful representations for downstream tasks \cite{bardes2021vicreg, radford2021learning, zbontar2021barlow,he2017mask,karpukhin2020dense}. More recently, the instance discrimination task has achieved state-of-the-art, and even exhibited to be competitive performance to supervised methods \cite{chen2020simple,chen2021empirical,gao2021simcse,dwibedi2021little}.

\subsection{Contrastive Self-Supervised Learning}
Contrastive instance discrimination originates from \citep{dosovitskiy2014discriminative,Wu_2018_CVPR}, whose core idea is to learn instance-invariant representations, i.e. each instance is viewed as a single class. The rational assumption behind is that maximizing similarities of the positive pairs and minimizing negative similarities can equip models with discrimination \cite{van2018representation}. To construct the negative sampling appropriately,  \citet{Wu_2018_CVPR,tian2020contrastive} and Moco family \cite{he2020momentum,chen2020improved} adopt extra structures to store negative vectors of instances. Instead, without additional parts for storing negative samples, other methods explore negative sampling within a large mini-batch, e.g., SimCLR \cite{chen2020simple}, CLIP \cite{radford2021learning}, and SimCSE \cite{gao2021simcse}. Some approaches successfully incorporate clusters or prototypes into CL \cite{caron2020unsupervised,huang2019unsupervised,dwibedi2021little,li2020prototypical}. It is also possible to learn only relying on positive samples \cite{grill2020bootstrap,chen2021exploring}, but InfoNCE-based contrastive methods are still the mainstream for various modalities and tasks \cite{afham2022crosspoint,gao2021simcse,radford2021learning,wang2021dense,pmlr-v162-li22v}.

\subsection{Contrastive InfoNCE Loss}
To understand the success of CL methods and enhance them, recent work has attempted to explore important properties of contrastive loss \cite{jing2021understanding}. InfoNCE is constructed by CPC \cite{van2018representation} and CMC \cite{tian2020contrastive} to maximize the mutual information of features from same the instance. Besides, some work focuses on the positive and negative pairwise similarity in InfoNCE. For example, \citet{wang2020understanding} attribute the effectiveness of InfoNCE to the asymptotical alignment and uniformity properties of features on hypersphere. Following this, \citet{wang2021understanding} have proven that the temperature parameter plays an essential role in controlling the penalty strength on negative samples, which is related to the hardness-aware property and a uniformity-tolerance dilemma. This temperature effect is also mentioned in \cite{chen2021simpler}. $\alpha$-CL \cite{tian2022deep} formulates InfoNCE as a coordinate-wise optimization, in which the pairwise importance $\alpha$ determines the importance weights of samples. 

Motivated by reducing the training batch size, DCL \cite{yeh2022decoupled} removes the positive similarity in the denominator of InfoNCE to eliminate a negative-positive-coupling effect. Furthermore, \citet{zhang2022dual} extend the hardness-aware property anchor-wise and introduce an extra larger temperature for InfoNCE. There are also some efforts in explicitly modeling false/hard negative samples in training to improve CL \cite{shah2022max,kalantidis2020hard}, e.g., HCL \cite{robinson2020hard} develops an importance sampling strategy to recognize true and false negatives. Our work mainly focuses on alleviating uniformity-tolerance dilemma and exploring the gradient reduction problem.

\section{Problem Definition}
\subsection{Contrastive Loss Function}
Let ${X}=\left\lbrace \boldsymbol{x}_{i}\right\rbrace^N_{i=1} $ denote the unlabeled training dataset. Also given encoders $ f $ and $ g $, instance $ \boldsymbol{x}_{i} $ is mapped to a query feature$ \boldsymbol{f}_i=f\left( \boldsymbol{x}_{i}\right) $ and a corresponding key feature $ \boldsymbol{g}_i=g\left( \boldsymbol{x}_{i}\right) $ on hypersphere with augmentations, respectively. $ g $ maybe a weight-shared network of  $ f $ or a momentum-updated encoder. Assume that the generated query (anchor) set and key set are denoted by ${F}=\left\{\boldsymbol{f}_{i}\right\}_{i=1}^{N}$ and ${G}=\left\{\boldsymbol{g}_{i}\right\}_{i=1}^{K+1}$, respectively, where $N$ is batch size and $K$ denotes the negative size. Then, the InfoNCE loss of the instance $ \boldsymbol{x}_{i} $ can be formulated as:
\begin{equation}
	\label{CL loss}
	\mathcal{L}_{\boldsymbol{x}_i}=-\log \frac{\exp \left(\boldsymbol{f}_i^\mathrm{T}\boldsymbol{g}_i / \tau\right)}{\exp \left(\boldsymbol{f}_i^\mathrm{T} \boldsymbol{g}_i / \tau\right)+\sum_{j=1}^{K} \exp 	\left(\boldsymbol{f}_i^\mathrm{T} \boldsymbol{g}_j / \tau\right)},
\end{equation}
where $ \left\lbrace \boldsymbol{f}_i, \boldsymbol{g}_i\right\rbrace  $ is the positive pair of the $ i $-th instance, and $ \boldsymbol{g}_j $ denotes the negative sample from a distinct instance. Temperature parameter is $\tau$ and $\tau>0$. Negative pairs can also be incremental from the same-side encoder like NT-Xent \cite{chen2020simple}. The final total loss of an iteration is the mean value on the mini-batch: $\mathcal{L}=\sum_{i=1}^{N}\mathcal{L}_{\boldsymbol{x}_i}/N$.

\subsection{Hardness-aware Property}
Previous work identifies the important hardness-aware property via gradient analysis. For convenience, let $ \mathcal{P}_{ij} $ indicate the similarity between $ x_{i} $ and $ x_{j} $ after scaled by temperature $\tau$ and Softmax operation:
\begin{equation}
	\label{P}
	\mathcal{P}_{ij}=\frac{\exp \left(\boldsymbol{f}_i^\mathrm{T} \boldsymbol{g}_j / \tau\right)}{\exp \left(\boldsymbol{f}_i^\mathrm{T}\boldsymbol{g}_i / \tau\right)+\sum_{r=1}^{K} \exp \left(\boldsymbol{f}_i^\mathrm{T} \boldsymbol{g}_r / \tau\right)},
\end{equation}
Then the gradient w.r.t the anchor $ \boldsymbol{f}_i $ can be formulated as follows (more details are show in Appendix \ref{grad_form}):
\begin{equation}
	\label{Grad}
	\frac{\partial \mathcal{L}_{\boldsymbol{x}_i}}{\partial \boldsymbol{f}_i}=-\frac{\mathcal{W}_i}{\tau}\left(\boldsymbol{g}_i-\sum_{j=1}^{K}\hat{\mathcal{P}}_{ij}\cdot \boldsymbol{g}_j\right),
\end{equation}
where $\mathcal{W}_i=\sum_{j=1}^{K} \mathcal{P}_{ij}$ can be seen as a gradient sacling factor, and there exists $\hat{\mathcal{P}}_{ij}=\mathcal{P}_{ij} / \sum_{r=1}^{K} \mathcal{P}_{ij}$. It is worth noting $ \sum_{j=1}^{K} \hat{\mathcal{P}}_{ij}=1 $, in which $ \hat{\mathcal{P}}_{ij} $ indicates an hardness-aware property. It implies that InfoNCE automatically puts larger penalty weights on the hard negatives \cite{wang2021understanding}, which are higher similar to the anchor sample.

\subsection{Uniformity-Tolerance Dilemma}
The weight on the negative sample $ x_j $ is formulated as:
\begin{equation}
	\label{p_hat}
	\hat{\mathcal{P}}_{ij}=\frac{\exp \left(\boldsymbol{f}_i^\mathrm{T} \boldsymbol{g}_j / \tau\right)}{\sum_{r = 1}^K \exp \left(\boldsymbol{f}_i ^\mathrm{T} \boldsymbol{g}_r / \tau\right)}, \quad i \neq j,
\end{equation}
which is controlled by the temperature parameter \cite{wang2021understanding}. (1) As $\tau$ decreases, the shape of $ \hat{\mathcal{P}}_{ij} $ becomes sharper. Thus, a smaller temperature causes larger penalties on the high similarity region, which encourages the separation of embeddings but has less tolerance for FNs. (2) A larger temperature makes the shape of $ \hat{\mathcal{P}}_{ij} $ flatter, then tends to give all negative samples equal magnitude of penalties. In this case, the optimization process is more tolerant to FNs while concentrating less on uniformity.

\section{Model-Aware Temperature Strategy}
The existence of the uniformity-tolerance dilemma leads to suboptimal embedding space and performance degradation of downstream tasks \cite{wang2021understanding}. Selecting an ideal temperature may be helpful, but it is not easy to get that balance. Instead, considering that the fixity of temperature prevents InfoNCE from focusing both on uniformity and potentially semantic relationships, we design an adaptive strategy for contrastive learning to mitigate the challenge.

\subsection{Adaptive to Alignment}
\label{adapt_to_A}
The uniformity-tolerance dilemma is rooted in the unsupervised instance discrimination task. Intuitively, the discrimination of a model will be gradually improved along with the training process, then the high similarity region is more likely to contain FNs. A dynamic temperature that changes according to iterations might deal with UTD better. \textit{However, since the training iteration does not reflect the level of semantic confidence for a CL model, such temperature strategies are still rough and heuristic by now.} The more reasonable temperature adjustment strategy is needed to be investigated. What motivates us is the alignment property of the embedding space.

Alignment property is one of the most critical prior assumptions for instance discrimination \cite{wang2020understanding,Wu_2018_CVPR,ye2019unsupervised}. It means that the representations from a positive pair should have high similarity. Since there are no labels available, it is impossible for SSL to explicitly construct semantic guidance. Instead, different views of the same instance are exploited for self-supervised learning. Alignment represents the awareness of view-invariance of a CL model, which is the base for exploring semantically consistent samples. \citet{wang2020understanding} formulate the alignment loss as the expected distance of positive pairs:
    \begin{equation}
        \label{align loss}
        \mathcal{L}_\text{align} =\underset{\boldsymbol{x}_{i} \sim X}{\mathbb{E}}\left[\|f(\boldsymbol{x}_{i})-g(\boldsymbol{x}_{i})\|_{2}^{2}\right].
    \end{equation}
    
Another significant thing is that estimating the magnitude of alignment is not a computationally expensive operation. As shown in Eqn.(\ref{CL loss}), the calculation of sample similarities is a required step for CL loss, in which the part of positive pairs can be directly exploited for alignment. In this paper, we define the alignment magnitude $\mathcal{A}$ as the expected similarity of positive pairs. Hence, no additional structures and computations are needed. Here exists:
    \begin{equation}
    \label{align define}
        \begin{aligned}
            \mathcal{A}&=\underset{\boldsymbol{x}_{i} \sim X}{\mathbb{E}}\left[ f(\boldsymbol{x}_{i})^\mathrm{T} g(\boldsymbol{x}_{i})\right] \\
            &=1-\frac{1}{2}\underset{\boldsymbol{x}_{i} \sim X}{\mathbb{E}}\left[\|f(\boldsymbol{x}_{i})-g(\boldsymbol{x}_{i})\|_{2}^{2}\right].
        \end{aligned}
    \end{equation}
Thus, we have $\mathcal{A}=1-\mathcal{L}_\text{align}/2$ for alignment (detailed in Appendix \ref{align_proof}). $\mathcal{A}=1$ implies perfect alignment.
\subsection{Implementation Details}
Then the proposed alignment-adaptive temperature strategy is formulated as:
\begin{equation}
	\label{ta}
    \begin{aligned}
        \tau_a&=\tau_0+\alpha\left(\underset{\boldsymbol{x}_{i} \sim X}{\mathbb{E}}\left[ f(\boldsymbol{x}_{i})^\mathrm{T} g(\boldsymbol{x}_{i})\right]-\mathcal{A}_0\right)\tau_0\\
        &=\left[1+\alpha (\mathcal{A}-\mathcal{A}_0)\right]\tau_0,
    \end{aligned}
\end{equation}
where $ \alpha $ is a scaling factor and $\alpha\in\left[ 0,1\right]$. $\mathcal{A}_0$ is a initial threshold for alignment magnitude. On the unit hypersphere, $ \boldsymbol{f}_i^\mathrm{T} \boldsymbol{g}_i\in \left[ -1,1\right]  $, then $ \tau_a \in \left[ (1-\alpha-\alpha \mathcal{A}_0)\tau_0,(1+\alpha-\alpha \mathcal{A}_0)\tau_0\right] $. \textit{In particular, iff $\alpha=0$, the temperature degenerates to the ordinary fixed case.} $\tau_a$ will be detached by stop gradient operation. The above form ensures the temperature changes in a proper range. In fact, lots of variants can be explored, but being alignment-adaptive is the most important point.

Eqn.(\ref{ta}) shows that the $\tau_a$ is an increasing function of $\mathcal{A}$, enabling the temperature to be adaptive to the alignment magnitude of the CL model during training. Specifically, a smaller temperature works when the model is lacking training by heavily penalizing those HNs. For the better-trained stage, the improved alignment indicates the CL model is more discriminative for samples. Naturally, larger temperature parameters can relax the penalty strength on the high similarity region, where is more possible to exist FNs now.

The proposed strategy is a \textit{fine-grained adjusting approach}. As CL models are trained by sampling mini-batches, $ \mathcal{A}$ can be estimated within a batch to promptly adjust the temperature. Thus, $\tau_a$ automatically adapts to the model of $t$-th optimization iteration. Compared with the one that simply increases by epochs, our adaptive strategy is more online. Thus, the proposed method is a Model-Aware Contrastive Learning (MACL) strategy.

\section{Gradient Reduction Dilemma}
\label{Gradient Reduction}
With the above temperature strategy, the improved CL loss helps to escape UTD. However, the $\mathcal{P}_{ij}$ term also impedes efficient contrastive learning in another aspect. The problem is that CL models are typically trained with a large number $K$ of negative samples to achieve better performance, which is computationally demanding, especially for large batch sizes. Some recent work tries to address this problem by modifying InfoNCE loss \textit{but they each have their own opinions} \cite{yeh2022decoupled,zhang2022dual,chen2021simpler}, whereas we prove they fall into a similar solution targeting the gradient reduction dilemma, but also summarily propose a simple reweighting method.
\begin{figure}[ht]
	\center
	\includegraphics[width=6.5cm]{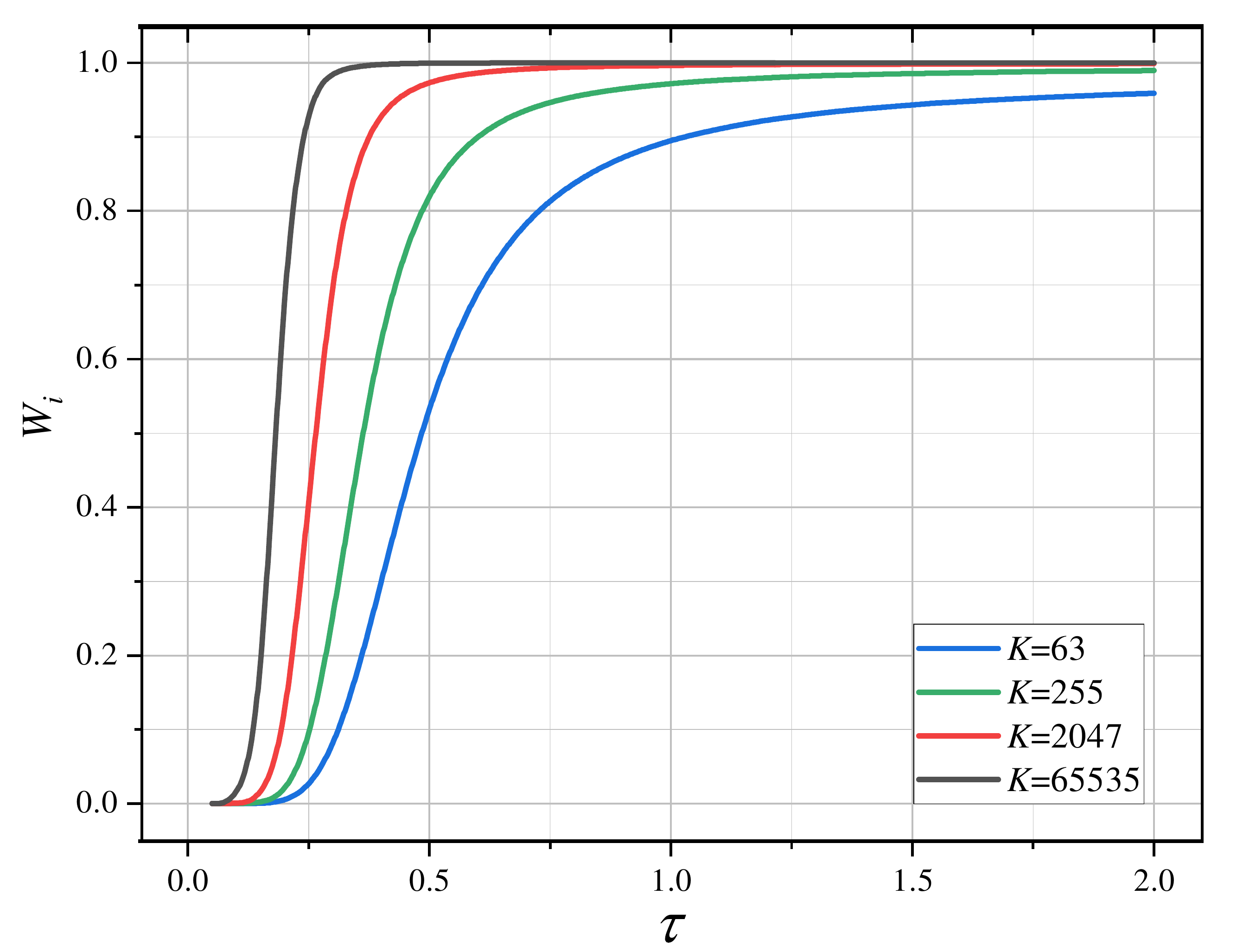}
	\caption{\textbf{Effect of the $ \tau $ and $K$} on the gradient scaling factor $ \mathcal{W}_i $.}
	\label{Wi}
\end{figure}
\subsection{Gradient Reduction Caused by the Sum Item}
The gradient scaling factor $ \mathcal{W}_i $ is a sum item of $ \mathcal{P}_{ij} $ in Eqn.(\ref{Grad}) and can also be described as:
\begin{equation}
	\mathcal{W}_i=  1-\frac{\exp \left(\boldsymbol{f}_i^\mathrm{T} \boldsymbol{g}_i / \tau\right)}{\exp \left(\boldsymbol{f}_i^\mathrm{T} \boldsymbol{g}_i / \tau\right)+\sum_{j=1}^{K} \exp 	\left(\boldsymbol{f}_i^\mathrm{T}\boldsymbol{g}_j / \tau\right)}.
\end{equation}
This item has small values for those easy positive pairs, which will reduce the gradient in Eqn.(\ref{P}) and has been mentioned in \cite{yeh2022decoupled}. Therefore, the gradient reduction problem will hinder the model learning, especially for those deeper units in low-precision ﬂoating-point training with the chain rule. In addition, a smaller $ K $ leads to a significant gradient reduction as there is an insufficient accumulation of negative similarities. This is the intuitive rationale that state-of-the-art CL models are often trained with a large number of negative samples.

From another aspect, $ \mathcal{W}_i $ is a monotonic function of $ \tau $. In particular, the shape of the sum item tends to become flat as temperature increases. We present an extreme example in Fig. \ref{Wi}, in which the similarities of the positive pair and negative pairs are set to 1 and -1, respectively. For these analyses, we have the following propositions (please check Appendix \ref{propo_proof} for proof details):
\begin{proposition}[Bound of gradient scaling factor w.r.t $ K $]
	\label{w bound1}
	Given the anchor feature $ \boldsymbol{f}_i $ and temperature $ \tau $, if $ K \rightarrow +\infty $, then $ \mathcal{W}_i $ approaches its upper bound 1. The limit is formulated as:
	\begin{equation}
		\lim _{K \rightarrow+\infty}\mathcal{W}_i = 1.
	\end{equation}
\end{proposition}
\begin{proposition}[Bound of gradient scaling factor w.r.t $ \tau $]
	\label{w bound2}
	Given $ \boldsymbol{f}_i $ and key set $ G $, $\mathcal{W}_i$ monotonically changes with respect to $ \tau$. The monotonicity is determined by the similarity distribution of samples. If $ \tau \rightarrow +\infty $, then $ \mathcal{W}_i $ approaches its bound $ K/(K+1) $, formulated as:
	\begin{equation}
			\lim _{\tau \rightarrow+\infty}\mathcal{W}_i = \frac{K}{1+K}.
	\end{equation}
\end{proposition}
\subsection{Discussion about Previous Studies}
These explorations show that the gradient reduction dilemma can be addressed by increasing the number of negative keys or adopting an extra large temperature for $ \mathcal{W}_i $. More specifically, sampling more negative keys helps to promote the accumulation of the exponential similarities and then inhibit $ \mathcal{W}_i $ too small. This is one of the reasons that most InfoNCE-based CL methods benefit from a large $K$, whether a big batch \cite{chen2020simple,dwibedi2021little} or a large dictionary size \cite{he2020momentum,tian2020contrastive}. In another case, adopting a larger separate temperature makes $ \mathcal{W}_i $ approach its bound then also improves this issue, which is the key of \cite{zhang2022dual}. Additionally, DCL \cite{yeh2022decoupled} removes the positive term from the denominator, then the corresponding gradient does not include $\mathcal{W}_i$ anymore. FlatNCE \cite{chen2021simpler} has exactly the same gradient expression with DCL, thus it is also feasible. We will recall their relations and provide some experimental evidence in Sec. \ref{rew_discu}.
\subsection{Reweighting InfoNCE with Upper Bound}
The above analysis essentially explains why some previous work is experimentally effective. We also design an approach for the gradient reduction issue when learning with small negative sizes, which is formulated as follows:
\begin{equation}
	\label{loss_m}
	\begin{aligned}
		\mathcal{L}_{\boldsymbol{x}_i}^{M}&=-\mathcal{V}_i\cdot\log \frac{\exp ( \boldsymbol{f}_i^\mathrm{T} \boldsymbol{g}_i / \tau_a) }{\sum_{j=1}^{K+1} \exp ( \boldsymbol{f}_i^\mathrm{T} \boldsymbol{g}_j /  \tau_a)},
	\end{aligned}
\end{equation}
where $\mathcal{V}_i=\operatorname{sg}\left[{1}/{\mathcal{W}_i}\right]$ and $\operatorname{sg}\left[\cdot\right]$ is the stop gradient operation to maintain the basic assumptions of InfoNCE, which is commonly used and finished by \texttt{detach} in code. In this case, the $ \mathcal{W}_i $ item is rearranged with 1 for the small $K$ cases Eqn.(\ref{Grad}), assigned to its upper bound directly. A simple example pseudocode of Eqn.(\ref{loss_m}) is shown as Algorithm \ref{sim_code}.

\begin{algorithm}[H]
	\caption{Pseudocode of MACL in a PyTorch-like style.}
	\label{sim_code}
	\definecolor{codeblue}{rgb}{0.25,0.5,0.5}
    \definecolor{codekw}{rgb}{0.85, 0.18, 0.50}
    \newcommand{\algofontsize}{8.0pt}
    \lstset{
      backgroundcolor=\color{white},
      basicstyle=\fontsize{\algofontsize}{\algofontsize}\ttfamily\selectfont,
      columns=fullflexible,
      breaklines=true,
      captionpos=b,
      commentstyle=\fontsize{\algofontsize}{\algofontsize}\color{codeblue},
      keywordstyle=\fontsize{\algofontsize}{\algofontsize}\color{codekw},
    }
	\centering
\begin{lstlisting}[language=python]
# pos: positive similarities, Nx1
# neg: negative similarities, NxK
# t_0: basic temperature
# a: scaling factor
# A_0: initial alignment threshold
    
def MACL(pos, neg, t_0, a, A_0):
    
    # model-aware temperature
    A = torch.mean(pos.detach())
    t = t_0 * (1 + a * (A - A_0))
    
    logits = torch.cat([pos, neg], dim=1)
    P = torch.softmax(logits/t, dim=1)[:, 0]
    
    # reweighting the loss
    V = 1 / (1 - P)
    loss = -V.detach() * torch.log(P)
    
    return loss.mean()
\end{lstlisting}
\end{algorithm}

\renewcommand{\arraystretch}{0.85}

\section{Empirical Study}
In this section, we empirically evaluate the proposed strategy for enhancing CL performance in different cases. To demonstrate the general improvement, experiments are mainly implemented on the learning of images, but also include sentences and graph representations.
\subsection{Experiments on Image Representation}
\label{img_exp}
We mainly experiment on the ImageNet ILSVRC-2012 (i.e., ImageNet-1K) \cite{deng2009imagenet} and use standard ResNet-50 \cite{he2016deep} as image encoders. CIFAR-10 \cite{krizhevsky2009learning}  and the subset ImageNet-100 \cite{tian2020contrastive} are also considered. We choose SimCLR \cite{chen2020simple} as the baseline but also perform some MoCo v2 \cite{chen2020improved} evaluations. They use InfoNCE (or NT-Xent) as the basic schedule and are representative of mainstream frameworks, sampling negatives within mini-batches, from a momentum queue, respectively. We strictly follow their settings, augmentations, and linear evaluation protocol or reproduce under the same standard. As such, comparisons are solely on loss function impact. Details are laid out in Appendix \ref{Setup Details}.
\begin{figure}[htb]
	\center
	\includegraphics[width=8cm]{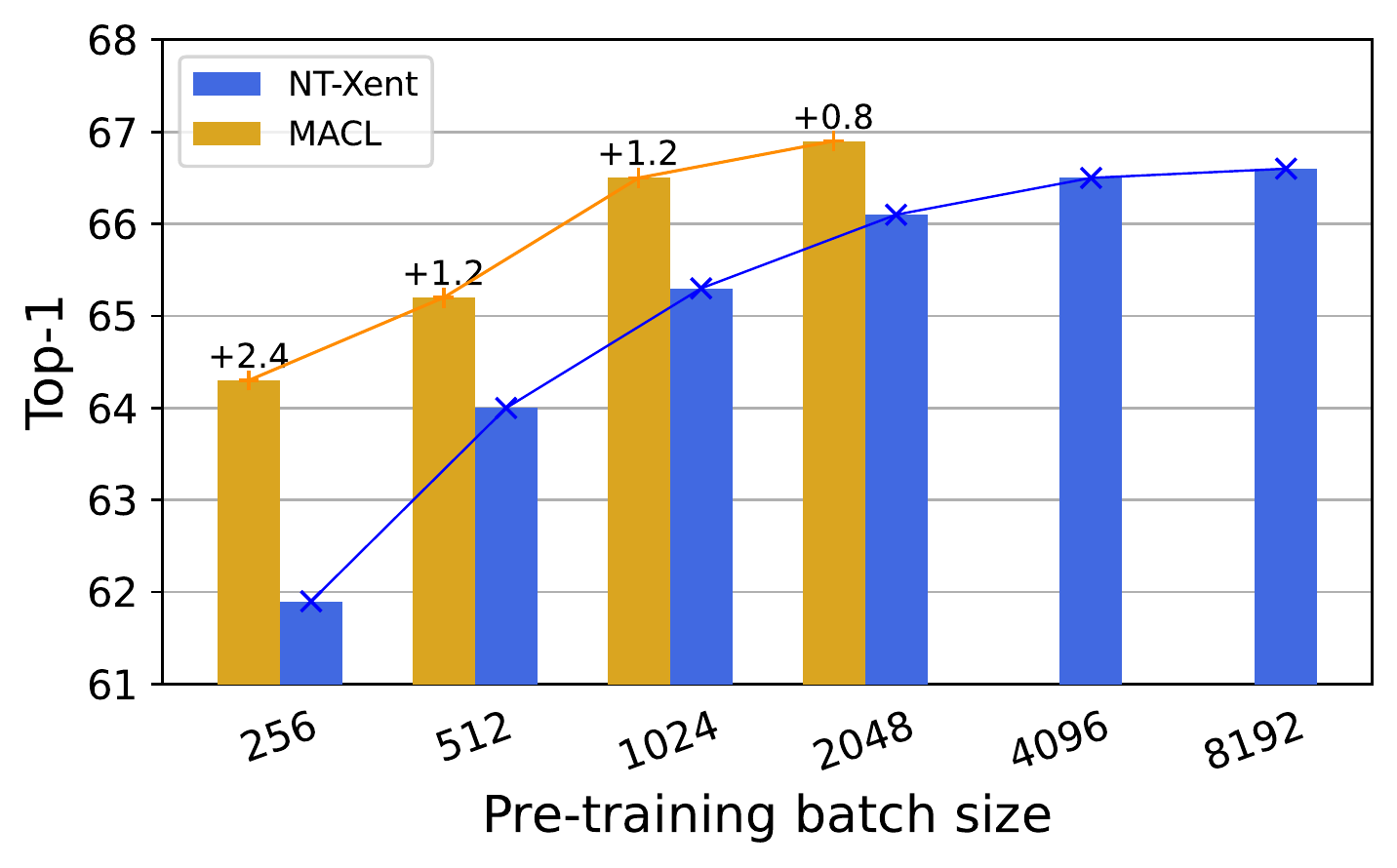}
	\caption{\textbf{Effect of batch sizes} (top-1 linear evaluation accuracies on ImageNet-1K with 200-epoch pre-training). Numbers on the top of bars are absolute gains of MACL under same settings.}
	\label{simclr_bs}
\end{figure}

\label{neg_size}
\textbf{Effect of Negative Sizes}\quad First, we compare MACL against vanilla CL loss for negative sizes. Figure \ref{simclr_bs} recapitulates the results of SimCLR and MACL with batch sizes from 256 to 2048. From these linear evaluation scores, we can see that encoders trained with MACL significantly outperform the vanilla versions (NT-Xent) under all the negative sizes, and the accuracy under 256-batch size is higher than the 512 one of the counterpart. In fact, \textit{our accuracy 66.5\% under 1024-batch size is on par with the original 8192 one (66.5\% vs 66.6\%)}, which indicates the effectiveness for MACL strategy to escape the dilemmas.

Affected by the gradient reduction problem discussed in Sec. \ref{Gradient Reduction}, SimCLR has a 4.2\% drop from batch size of 2048 to 256. With MACL, the trained encoders are less sensitive to batch size as have a smaller corresponding 2.6\% drop, and have higher improvement under smaller batch sizes. Besides, comparisons and discussion of queue size with MACL and InfoNCE on ImageNet-100 with MoCo v2 are reported in Appendix \ref{queue_size_exp}. These results suggest the rationality of the gradient reduction dilemma analysis and reweighting approach for alleviating it.

\textbf{Robustness to Training Length}\quad We conduct longer training with MACL, and the linear classification accuracies are shown in Table \ref{in1k_epoch}. There are some observations. First, MACL benefits from longer training length, which is consistent with vanilla contrastive loss. Moreover, MACL-based results are significantly better than ordinary ones. Our 200 and 400 epochs accuracies based on SimCLR are even comparable to the original ones with twice epochs (400 and 800), which demonstrates the learning efficiency brought by MACL. This also validates the advantage of MACL for dealing with the underlying dilemmas in InfoNCE.
\begin{table}[htb]
\setlength{\tabcolsep}{11pt}
  \centering
  \caption{\textbf{Effect of training lengths} (top-1 linear evaluation accuracies on ImageNet1K with 256-batch size pre-training).}
    \scalebox{0.89}{\begin{tabular}{lccc}
    \toprule[1.2pt]
    \multicolumn{1}{l}{Epoch} & 200   & 400   & 800 \\
    \midrule
    SimCLR       & 61.9 & 64.7  & 66.6 \\
    \textbf{w/ MACL}  & \textbf{64.3}(\textcolor[RGB]{116,0,0}{+2.4}) &  \textbf{66.3}(\textcolor[RGB]{116,0,0}{+1.6})    & \textbf{68.1}(\textcolor[RGB]{116,0,0}{+1.5}) \\
    \bottomrule[0.9pt]
    \end{tabular}}
  \label{in1k_epoch}%
\end{table}

\textbf{Transfer to Object Detection}\quad We evaluate representations learned by MACL on downstream detection task. We use VOC07+12 \cite{everingham2010pascal} to finetune the encoders of SimCLR and MACL, then test models on VOC2007 test benchmark. Scores in Table \ref{simclr_detect} indicate that MACL strategy can provide better performance in terms of mean average precision metrics, demonstrating its effectiveness for learning transferable representations to detection.

\begin{table}[htb]
\setlength{\tabcolsep}{11pt}
  \centering
  \caption{\textbf{Transfer to object detection} on VOC07+12 using Faster R-CNN with C4-backbone and 1$\times$ schedule. Encoders are trained with batch size of 256.}
    \scalebox{0.89}{\begin{tabular}{lccc}
    \toprule[1.2pt]
    Pre-train & AP$_\text{all}$  & AP$_\text{50}$  & AP$_\text{75}$ \\
    \midrule
    SimCLR   & 49.7  & 79.4  & 53.6 \\
    \textbf{w/ MACL} & \textbf{50.1}(\textcolor[RGB]{116,0,0}{+0.4})  & \textbf{79.7}(\textcolor[RGB]{116,0,0}{+0.3})  & \textbf{53.7}(\textcolor[RGB]{116,0,0}{+0.1}) \\
    \bottomrule[0.9pt]
    \end{tabular}}
  \label{simclr_detect}%
\end{table}

\subsection{Experiments on Sentence Embedding}
We adopt SimCSE \cite{gao2021simcse} as the baseline in this part, which successfully facilities sentence embedding learning with contrastive learning framework using InfoNCE. The datasets and setups of training and evaluation follow the original literature and are detailed in Appendix \ref{simcse_setting}. Results under RoBERTa \cite{liu2019roberta} backbone are reported in Table \ref{roberta_sts_trans}, and BERT \cite{kenton2019bert} scores are listed in Appendix Table \ref{bert_sts_trans}.

\textbf{Performance on STS Tasks}\quad We conduct seven semantic textual similarity (STS) tasks to evaluate the capability of sentence embedding following \cite{gao2021simcse}. The results are measured by Spearman’s correlation. For both models with RoBERT and BERT backbones, those trained with the MACL strategy achieve better performance on 6 of 7 STS tasks. Additionally, there are also noticeable gains w.r.t the average score. With the help of MACL, the learned embeddings are able to boost the clustering of semantically similar sentences.
\begin{table}[htb]
\setlength{\tabcolsep}{2.3pt}
  \centering
  \caption{\textbf{STS and transfer tasks comparisons} of sentence embeddings with RoBERTa encoder. }
    \scalebox{0.8}{\begin{tabular}{ccccccccc}
    \toprule[1.2pt]
    STS task & STS12 & STS13 & STS14 & STS15 & STS16 & STSB  & SICKR  \\
    \midrule
    SimCSE & 70.16  & \textbf{81.77 } & 73.24  & 81.36  & 80.65  & 80.22  & 68.56   \\
    \multirow{2}{*}{\textbf{w/ MACL}} & \textbf{70.76 } & 81.43  & \textbf{74.29 } & \textbf{82.92 } & \textbf{81.86 } & \textbf{81.17 } & \textbf{70.70 } \\
     & (\textcolor[RGB]{116,0,0}{+0.60}) & (\textcolor[RGB]{34,139,34}{-0.34}) & (\textcolor[RGB]{116,0,0}{+1.05}) & (\textcolor[RGB]{116,0,0}{+1.56}) &(\textcolor[RGB]{116,0,0}{+1.21}) & (\textcolor[RGB]{116,0,0}{+0.95}) & (\textcolor[RGB]{116,0,0}{+2.14}) \\
    \midrule
    \midrule
    Transfer task & MR    & CR    & SUBJ  & MPQA  & SST2  & TREC  & MRPC  \\
    \midrule
    SimCSE & 81.04  & 87.74  & 93.28  & 86.94  & 86.60  & 84.60  & 73.68  \\
    \multirow{2}{*}{\textbf{w/ MACL}} & \textbf{82.32 } & \textbf{88.03 } & \textbf{93.51 } & \textbf{87.92 } & \textbf{87.81 } & \textbf{85.80 } & \textbf{75.54 } \\
     & (\textcolor[RGB]{116,0,0}{+1.28}) & (\textcolor[RGB]{116,0,0}{+0.29}) & (\textcolor[RGB]{116,0,0}{+0.23}) & (\textcolor[RGB]{116,0,0}{+0.98}) & (\textcolor[RGB]{116,0,0}{+1.21}) & (\textcolor[RGB]{116,0,0}{+1.20}) & (\textcolor[RGB]{116,0,0}{+1.86}) \\
    \bottomrule[0.9pt]
    \end{tabular}}
  \label{roberta_sts_trans}%
\end{table}

\textbf{Performance on Transfer Tasks}\quad We further investigate transfer tasks following \cite{gao2021simcse} to verify the superiority of transferring to downstream settings. A logistic regression classifier is trained on top of the frozen pre-trained models. From the exhibited evaluation scores, it can be observed that the model trained with MACL achieves superior results on all the tasks and obtain 1.01$\%$ gain w.r.t the average score. In the BERT context, our MACL strategy outperforms on 5/7 tasks over the original SimCSE and also shows superiority in the average score. More experimental details are described in Appendix \ref{simcse_setting}. Results both on STS and transfer tasks fully suggest that the proposed MACL strategy provides higher quality representations, then gives considerable improvement for sentence embedding learning.

\begin{table}[htb]
\setlength{\tabcolsep}{10pt}
  \centering
  \caption{\textbf{Downstream classification accuracies in graph representation learning} on different datasets.}
    \scalebox{0.89}{\begin{tabular}{cccc}
    \toprule[1.2pt]
    Dataset & NCI1  & PROTEINS & MUTAG \\
    \midrule
    GraphCL & 77.87±0.41 & 74.39±0.45 & 86.80±1.34 \\
    \textbf{w/ MACL} & \textbf{78.41±0.47} & \textbf{74.47±0.85} & \textbf{89.04±0.98} \\
    \midrule
    \midrule
    Dataset & RDT-B & DD    & IMDB-B \\
    \midrule
    GraphCL & 89.53±0.84 & 78.62±0.40 & 71.14±0.44 \\
    \textbf{w/ MACL} & \textbf{90.59±0.36} & \textbf{78.80±0.66} & \textbf{71.42±1.05} \\
    \bottomrule[0.9pt]
    \end{tabular}}
  \label{graph cls}
\end{table}
\subsection{Experiments on Graph Representation}
To evaluate on graph representation learning, we choose GraphCL \cite{you2020graph} as the baseline and compare MACL against ordinary CL loss on various benchmarks. The pre-training and evaluation settings are the default of GraphCL detailed in Appendix \ref{graph_trans}.

\textbf{Downstream Classification}\quad For the graph classification task, we conduct experiments on six commonly used benchmarks \cite{Morris+2020}. They are denser or not-so-dense and cover areas of the social network, bioinformatics data, molecules data, etc. GNN-based encoders are the same in \cite{chen2019powerful}. Methods are trained with contrastive strategies, and the generated graph embeddings are fed into a downstream SVM classiﬁer then reporting the mean and standard deviation values of five times following \cite{you2020graph}. As the scores shown in Table \ref{graph cls}, our MACL strategy enables the framework to achieve better or comparable performance on these six different-scale (number of average nodes) datasets belonging to distinct fields.

\begin{table}[htb]
\setlength{\tabcolsep}{3pt}
  \centering
  \caption{\textbf{Transfer learning comparisons of graph representation learning} on different datasets.}
  \scalebox{0.89}
    {\begin{tabular}{ccccccccc}
    \toprule[1.2pt]
    Dataset & Tox21 & BBBP  & ToxCast & SIDER \\
    \midrule
    GraphCL & 73.87±0.66 & \textbf{69.68±0.67} & 62.40±0.57 & 60.53±0.88 \\
    \textbf{w/ MACL}  & \textbf{74.39±0.29} & 67.98±0.97 & \textbf{62.96±0.28} & \textbf{61.46±0.39} \\
    \midrule
    \midrule
    Dataset & ClinTox & MUV   & HIV   & BACE \\
    \midrule
    GraphCL & 75.99±2.65 & 69.80±2.66 & \textbf{78.47±1.22} & 75.38±1.44 \\
    \textbf{w/ MACL} & \textbf{78.13±4.29} & \textbf{72.77±1.25} & 77.56±1.12 & \textbf{76.07±0.90} \\
    \bottomrule[0.9pt]
    \end{tabular}}
  \label{graph_transfer}
\end{table}

\textbf{Transfer to Chemistry Data}\quad Transfer learning comparisons are also considered. We experiment on molecular property prediction in chemistry following \cite{you2020graph, hu2020strategies}. Pre-trains and ﬁnetunes are in different datasets \cite{wu2018moleculenet}. Models trained with MACL outperform original GraphCL on 6 of 8 datasets in Table \ref{graph_transfer}. Both the downstream classification task and transfer learning results illustrate that MACL can boost representations with better generalizability and transferability, which further verifies the general improvement for vanilla CL loss.

\subsection{Ablations}
We present ablations of the proposed approach in this section to further understand its effectivity. Unless otherwise stated, settings are the same as those in Sec. \ref{img_exp}.

\begin{table}[htb]
\setlength{\tabcolsep}{5pt}
  \centering
  \caption{\textbf{Explorations of loss function}. Numbers are top-1 linear evaluation accuracies on ImageNet-1K with 200-epoch pre-training under 512-batch size. LR-s denotes the smaller learning rate case under the ordinary schedule, and LR-l is the larger case.}
    \scalebox{0.89}{\begin{tabular}{ccc|cc}
    \toprule[1.2pt]
    case  & Adaptive & Reweighting & LR-s & LR-l \\
    \hline
    Baseline & \ding{56}  & \ding{56} & 64.0  & 65.6 \\
    \hline
    (a)   & \ding{52}   & \ding{56}   & 64.9(\textcolor[RGB]{116,0,0}{+0.9})  & 67.5(\textcolor[RGB]{116,0,0}{+1.9}) \\
    (b)   & \ding{56}   & \ding{52}   & 65.0(\textcolor[RGB]{116,0,0}{+1.0})  & 67.8(\textcolor[RGB]{116,0,0}{+2.2}) \\
    (c)   & \ding{52}   & \ding{52}   & 65.2(\textcolor[RGB]{116,0,0}{+1.2})  & 68.1(\textcolor[RGB]{116,0,0}{+2.5}) \\
    \bottomrule[0.9pt]
    \end{tabular}}
  \label{loss_lr_abla}
\end{table}

\textbf{Loss Function Ablations}\quad To test the necessity of major components, we alter the loss function present in Eqn.(\ref{loss_m}) and validate encoders trained by variants. Linear evaluation scores are listed in Table \ref{loss_lr_abla} (the column of LR-s case). First, we can see that removing the adaptive temperature or reweighting operation leads to an accuracy drop compared to the full version. On the other hand, the model-aware adaptive method is designed to alleviate the performance degradation caused by the uniformity-tolerance dilemma. Utilizing this strategy in isolation yields a performance spike over the baseline. Since reweighting is designed to verify and improve the gradient reduction dilemma, only using this operation also achieves better performance. These observations support our motivation and designs. \citet{chen2020simple} show that SimCLR with a different learning rate schedule can improve the performance for models trained with small batch sizes and in smaller number of epochs. Interestingly, our MACL shows even higher improvement using a larger learning rate, which is present in the LR-l case in Table \ref{loss_lr_abla}. More discussions are in Appendix \ref{abla2}.

\begin{table}[htb]
\setlength{\tabcolsep}{2.5pt}
  \centering
  \caption{\textbf{Ablation comparisons on ImageNet-100 with  SimCLR framework} (linear evaluation accuracies with 200-epoch pre-training and batch size of 256).}
  \scalebox{0.89}
    {\begin{tabular} {c|cccc}
    \toprule[1.2pt]
    \multirowcell{2}{\centering Config} &{\multirowcell{2}{\centering NT-Xent}} &{\multirowcell{2}{\centering DCL}} & \multicolumn{2}{c}{\textbf{MACL}} \\
    \cline{4-5}
    & & & w/ adaptive   & w/o adaptive \\
    \hline
    Acc. & 75.54 / 93.06 & 77.38 / 94.01 & \textbf{78.28 / 94.25} & 77.32 / 94.03  \\
    \bottomrule[0.9pt]
    \end{tabular}}
  \label{alation_in100}%
\end{table}

For another, we compare MACL with NT-Xent and DCL in Table \ref{alation_in100}. When $\alpha$ is set to 0, the temperature reduces to the fixed case, and only reweighting works. We can see that the top-1 score has a 1.78 gain over NT-Xent using reweighting in isolation and is on par with DCL, which supports the correctness of our judgment about gradient reduction dilemma. Besides, when equipped with the adaptive temperature, MACL obtains a further 0.96 improvement.
\begin{figure}[htb]
	\center
	\subfigure[NT-Xent]{
		\label{NTXent_tsne}
		\includegraphics[width=3.8cm]{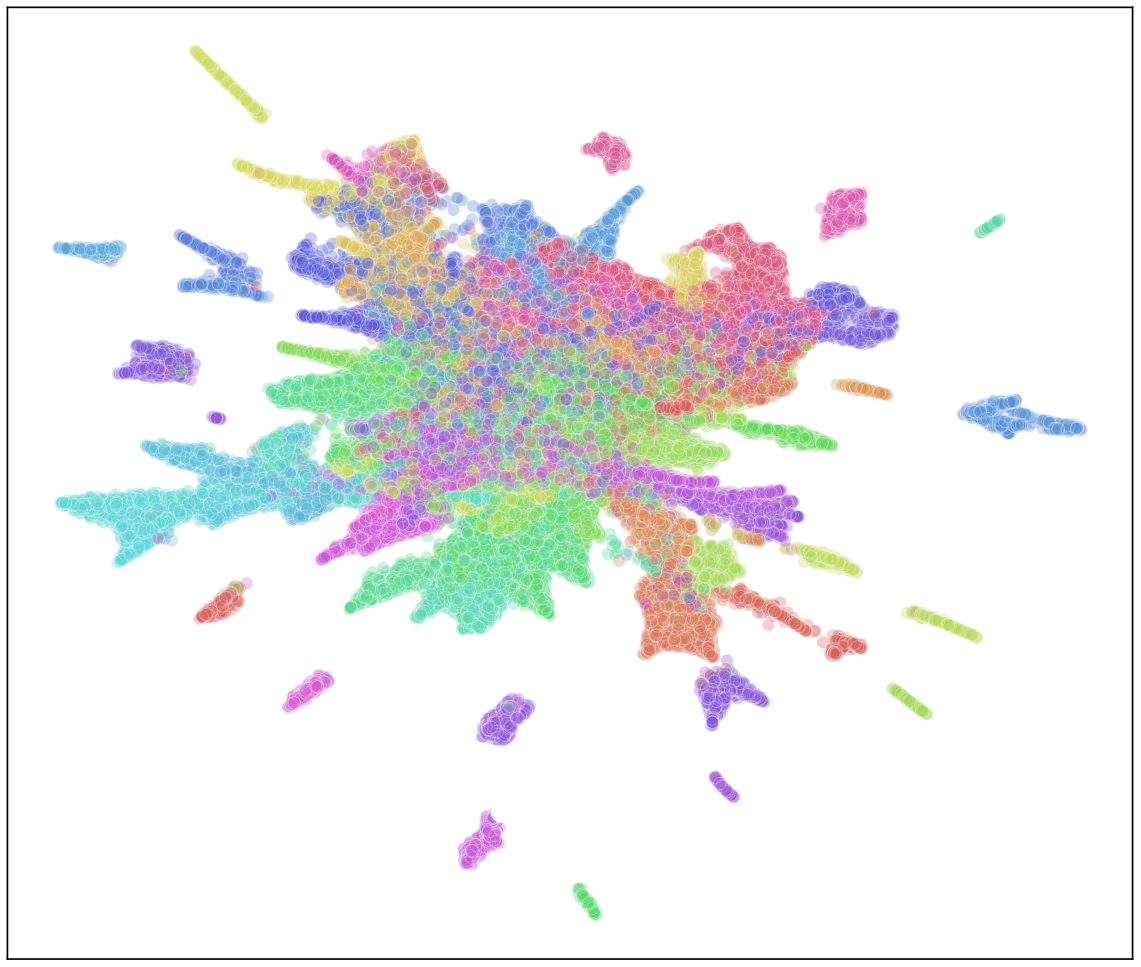}}
	\subfigure[MACL]{
		\label{macl_tsne}
		\includegraphics[width=3.8cm]{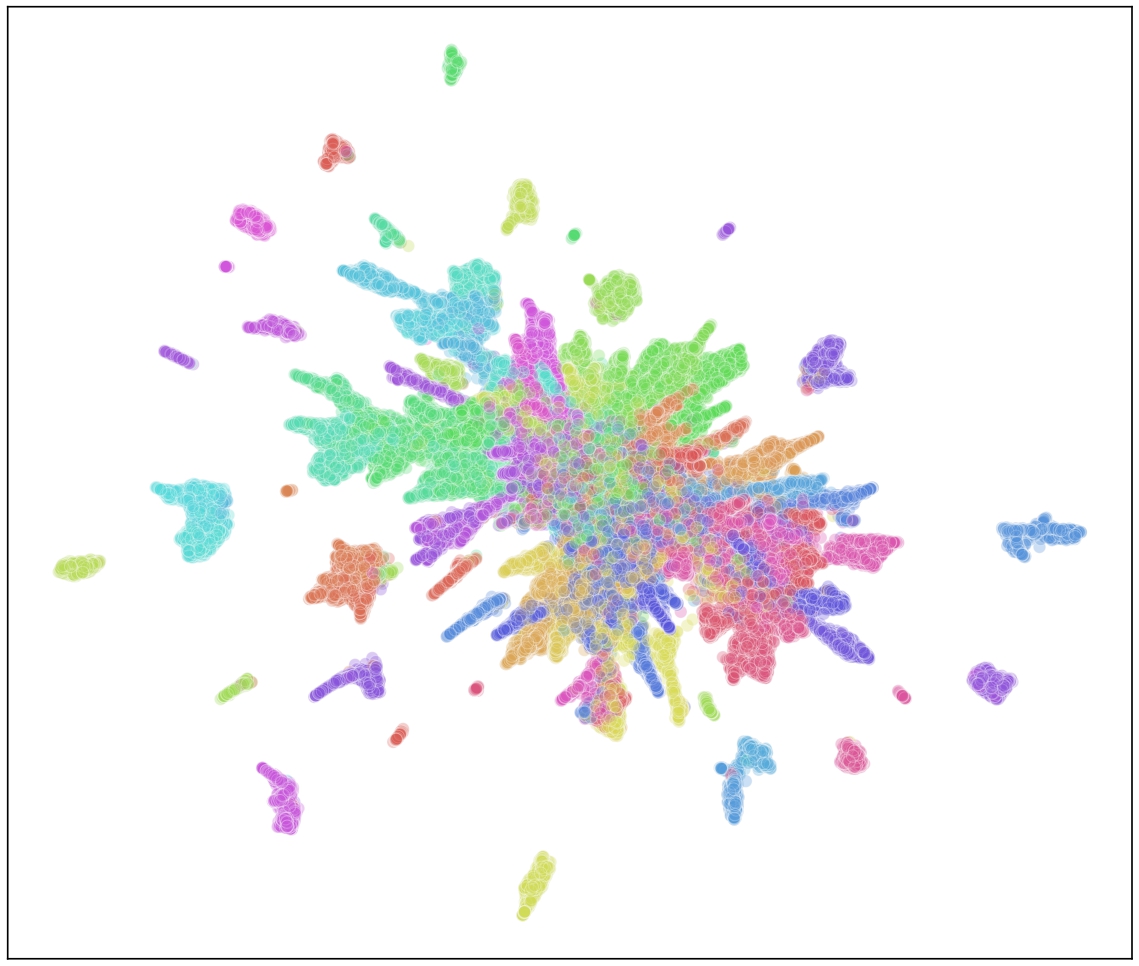}}
	\caption{\textbf{UMAP visualization comparison} on ImageNet-100. ResNet-50 encoders are pre-trained 200 epochs under the batch size of 256 with NT-Xent and MACL, respectively. There are 100 colors indicating 100 semantic categories.}
	\label{umap_vis}
\end{figure}

\textbf{More Ablations}\quad We have already presented some ablations in the former experimental sections. From Figure  \ref{simclr_bs} and Table \ref{in1k_epoch}, MACL exhibits significantly better robustness with respect to negative size and training length. We further present the UMAP visualization \cite{mcinnes2018umap} of features generated by encoders trained with our MACL and vanilla NT-Xent loss in Figure \ref{umap_vis}. Figure \ref{macl_tsne} exhibits better separability in the central area under the same training length, which indicates the learning efficiency and higher embedding quality brought by our approach.

\begin{table}[htb]
\setlength{\tabcolsep}{4pt}
  \centering
  \caption{\textbf{Parameter analysis} for MACL strategy (linear evaluation accuracies of 200-epoch and 256-batch size pre-training on ImageNet-100 with SimCLR). The underlined configs are set to be fixed when the other are selected to be variables.}
    \scalebox{0.85}{\begin{tabular}{ccccc}
    \toprule[1.2pt]
    $\tau_0$  & 0.05  & \underline{0.1}   & 0.5   & 1 \\
    \midrule
    Acc. & 76.68 / 93.46 & \textbf{78.14 / 94.16} & 69.72 / 91.71 & 61.72 / 87.28 \\
    \midrule
    \midrule
    $\alpha$ & 0     & 0.1   & \underline{0.5}   & 1 \\
    \midrule
    Acc. & 77.32 / 94.03   & 77.18 / 94.06      & \textbf{78.14 / 94.16} & 77.74 / 94.14 \\
    \midrule
    \midrule
    $\mathcal{A}_0$    & \underline{0}     & 0.2   & 0.6   & 0.8 \\
    \midrule
    Acc. & 78.14 / 94.16  & \textbf{78.28 / 94.25}  & 78.1 / 94.14      & 77.54 / 93.95 \\
    \bottomrule[0.9pt]
    \end{tabular}}
  \label{parameters}
\end{table}
\textbf{Parameter analysis}\quad To better understand MACL as well as its parameters, we conduct experiments on and the scores are listed in Table \ref{parameters}. Since $\tau_0$ is the datum point, it is essential to the performance of contrastive learning. Though the model is less sensitive to $\alpha$ and $\mathcal{A}_0$, they play an important role in adjusting the final temperature, yielding performance improvement with proper settings. The discussions of role of each parameter as well as the sensitivity analysis and ablations of NNCLR are present in Appendix \ref{simclr_in100}.

\section{Discussion}
\subsection{Relations to Recent Temperature Schemes}
Besides our alignment-adaptive strategy for addressing uniformity-tolerance dilemma, there are some interesting CL temperature schemes explored for different motivations. \citet{zhang2021temperature} aim to learn temperature as the uncertainty of embeddings for the out-of-distribution task. A dynamic multi-temperature method is proposed in \cite{khaertdinov2022dynamic} to scale instance-specific similarities in the Human Activity Recognition. The most recent \cite{kukleva2023temperature} designs temperature as a cosine variation with epoch to improve CL performance on long-tail data. Additionally, as mentioned in Sec. \ref{adapt_to_A}, designing the temperature as a function of the iteration may potentially aid in escaping from UTD, however, such methods are incapable of providing real-time feedback on the training status.

\begin{table}[htb]
\setlength{\tabcolsep}{9pt}
	\centering
	\caption{\textbf{Comparison of reweighting methods} (linear evaluation accuracies on CIFAR10, please check Appendix \ref{cifar_exp} for setting details and corresponding kNN results).}
	\scalebox{0.89}{\begin{tabular}{cccccc}
		\toprule[1.2pt]
		\multicolumn{1}{c}{Batch size} & 64    & 128   & 256   & 512   & 1024 \\
		\midrule
		NT-Xent & 82.31  & 83.56  & 84.65  & 85.13   & 85.30  \\
        \midrule
        FlatNCE & 86.30  & 86.28  & 86.11  & 86.02  & 85.84  \\
		DCL     & 86.28  & 86.04  & 86.29  & 86.33  & 85.61  \\
		Dual    & 86.32  & 86.40  & 85.86  & 86.23  & 86.05  \\
        \midrule
		\textbf{MACL} & \textbf{87.11} & \textbf{87.41}  & \textbf{87.27}  & \textbf{87.24}  & \textbf{86.75}  \\
		\bottomrule[0.9pt]
	\end{tabular}}
	\label{cifar10-batch}
\end{table}

\subsection{Relations with Previous Reweighting Methods}
\label{rew_discu}
As aforementioned, FlatNCE, DCL, and \cite{zhang2022dual} (Dual) essentially work against gradient reduction dilemma by approaching the bounds of the gradient scaling factor. Then we propose another feasible solution, reweighting the sum item with its upper bound directly. Furthermore, our MACL has an extra implicit alignment-adaptive reweighting for gradient of each step. For an under-optimization batch, the multiplier $1/\tau_a$ is bigger for Eqn.(\ref{Grad}) as the lower alignment causes smaller $\tau_a$, and vice versa. We test the performance of these methods. Results in Table \ref{cifar10-batch} show that all the related methods outperform vanilla NT-Xent, especially under smaller batch sizes. FlatNCE, DCL, and Dual perform on par. Since MACL has an adaptive temperature which can alleviate UTD, it shows further superiority.

\subsection{Contributions to $\alpha$-CL}
$\alpha$-CL \cite{tian2022deep} formulates InfoNCE loss as coordinate-wise optimization, in which each element $\alpha_{ij}$ of the min player $\alpha$ is the pairwise importance of $(i,j)$-pair that is equal to ${\mathcal{P}}_{ij}$. Our adaptive temperature actually provides an iteration-dynamic feasible set for $\alpha$, i.e., the landscape of constraint for $\alpha$ is different according to alignment magnitude. The entropy of $\alpha$ is a regularization for its min player, and will increase when the positive pairs are aligned better, since this entropy is a increasing function w.r.t $\tau$ \cite{wang2021understanding}. Furthermore, the constraint will reduce to a sample-agnostic case if the reweighting is applied.

\subsection{Relations to Hard Negative Sampling}
Hard negative sampling methods \cite{chuang2020debiased} attempt to alleviate the drawback of instance discrimination via explicitly modeling false or hard negative samples. Such approaches have achieved promising results and are formulated by probability \cite{robinson2020hard}, mixing\cite{kalantidis2020hard}, aggregation \cite{huynh2022boosting}, or using SVM for the decision hyperplane of negatives \cite{shah2022max}. Instead, our MACL also pays attention to negatives but has adaptive penalty strengths on them, which is model-aware for FNs.

\section{Conclusion}
In this work, we analyze InfoNCE and provide strategies to escape the underlying dilemmas. To alleviate the uniformity-tolerance dilemma, an alignment-adaptive temperature is designed. Besides, we offer some insights into the importance of the negative sample size and the temperature by analyzing gradient reduction. A new contrastive loss is exploited based on these strategies. Experiment results in three modalities verify the superiority of our MACL strategy for improving contrastive learning.

\section*{Acknowledgements}
We thank anonymous reviewers for their constructive comments. This work was partially supported by the National Natural Science Foundation of China (Nos. 62176116, 62276136,
62073160), and the Natural Science Foundation of the Jiangsu Higher Education Institutions of China under Grant
20KJA520006.

\bibliography{example_paper}
\bibliographystyle{icml2023}

\newpage
\appendix
\onecolumn
\counterwithin{figure}{section}
\counterwithin{table}{section}

\setcounter{proposition}{0}

\section{Proofs and Additional Analysis}
\subsection{Gradient of InfoNCE}
\label{grad_form}
Given sampled mini-batch of instances with $K$ negative samples, the InfoNCE loss of instance $\boldsymbol{x}_i$ is expressed as:
\begin{equation*}
	\mathcal{L}_{\boldsymbol{x}_i}=-\log \frac{\exp \left(\boldsymbol{f}_i^\mathrm{T}\boldsymbol{g}_i / \tau\right)}{\exp \left(\boldsymbol{f}_i^\mathrm{T} \boldsymbol{g}_i / \tau\right)+\sum_{j=1}^{K} \exp 	\left(\boldsymbol{f}_i^\mathrm{T} \boldsymbol{g}_j / \tau\right)}.
\end{equation*}
For simplicity, let $ E_k= \exp \left(\boldsymbol{f}_i ^\mathrm{T} \boldsymbol{g}_{k} / \tau\right)$, and $\mathcal{L}_{\boldsymbol{x}_i}$ is reformulated as:
\begin{equation*}
	\mathcal{L}_{\boldsymbol{x}_i}=-\log \frac{E_i}{E_i+\sum_{j=1}^{K} E_j}.
\end{equation*}
Then the gradient with respect to $ \boldsymbol{f}_i $ is:
\begin{equation*}
    \begin{aligned}
        \frac{\partial \mathcal{L}_{\boldsymbol{x}_i}}{\partial \boldsymbol{f}_i}=&\frac{-1}{\tau}\frac{\sum_{r=1}^{K}E_r}{E_i+\sum_{r=1}^{K} E_r}\cdot\left(\boldsymbol{g}_i-\sum_{j=1}^{K}\frac{E_j}{\sum_{k=1}^{K} E_k}\cdot \boldsymbol{g}_j\right).
    \end{aligned}
\end{equation*}
Let $ \mathcal{P}_{ij} $ denote
\begin{equation*}
	\mathcal{P}_{ij}=\frac{E_i}{E_i+\sum_{r=1}^{K} E_r},
\end{equation*}
$\mathcal{W}_i=\sum_{j=1}^{K} \mathcal{P}_{ij}$, and $\hat{\mathcal{P}}_{ij}=\mathcal{P}_{ij} / \sum_{r=1}^{K} \mathcal{P}_{ij}$, where $ \sum_{j=1}^{K} \hat{\mathcal{P}}_{ij}=1 $. Therefore, the above gradient can be reformulated as:
\begin{equation}
\label{grad_fi}
        \frac{\partial \mathcal{L}_{\boldsymbol{x}_i}}{\partial \boldsymbol{f}_i}=-\frac{\mathcal{W}_i}{\tau}\left(\boldsymbol{g}_i-\sum_{j=1}^{K}\hat{\mathcal{P}}_{ij}\cdot \boldsymbol{g}_j\right).
\end{equation}
Since the MoCo type algorithms detach the features in key set via a stop gradient operation, thus we discuss the loss function according to Eqn.(\ref{grad_fi}). For SimCLR type methods, we can also derive the corresponding gradient with respect to $ \boldsymbol{g}_i $:
\begin{equation}
    \frac{\partial \mathcal{L}_{\boldsymbol{x}_i}}{\partial \boldsymbol{g}_i} =-\frac{\mathcal{W}_i}{\tau}\cdot\boldsymbol{f}_i,
\end{equation}
and the gradient with respect to $ \boldsymbol{g}_j $:
\begin{equation}
     \frac{\partial \mathcal{L}_{\boldsymbol{x}_i}}{\partial \boldsymbol{g}_j} =\frac{\mathcal{W}_i}{\tau}\hat{\mathcal{P}}_{ij}\cdot\boldsymbol{f}_i.
\end{equation}

\subsection{Proof of Equation (\ref{align define})}
\label{align_proof}
\begin{proof}[Proof of $\mathcal{A}$]
Since representations $ \boldsymbol{f}_i=f\left( \boldsymbol{x}_{i}\right) $ and $ \boldsymbol{g}_i=g\left( \boldsymbol{x}_{i}\right) $ lie on a unit hypersphere ($\ell_2$ normalized after the last layer of encoders), i.e., $f,g:\mathbb{R}^d \rightarrow \mathcal{S}^{m-1}$, where $d$ and $m$ denote the dimension of data space and hypersphere feature space. For $f(\boldsymbol{x}_{i}),g(\boldsymbol{x}_{i})\in \mathcal{S}^{m-1}$, there exists: $f(\boldsymbol{x}_{i})^\mathrm{T} f(\boldsymbol{x}_{i})=g(\boldsymbol{x}_{i})^\mathrm{T} g(\boldsymbol{x}_{i})=1$, thus
    \begin{equation*}
        \|f(\boldsymbol{x}_{i})-g(\boldsymbol{x}_{i})\|_{2}^{2}=2-2f(\boldsymbol{x}_{i})^\mathrm{T} g(\boldsymbol{x}_{i}),
    \end{equation*}
then, the relation of alignment $\mathcal{A}$ and alignment loss $\mathcal{L}_\text{align}$ is derived as:
    \begin{equation*}
        \begin{aligned}
            \mathcal{A}&=\underset{\boldsymbol{x}_{i} \sim X}{\mathbb{E}}\left[ f(\boldsymbol{x}_{i})^\mathrm{T} g(\boldsymbol{x}_{i})\right] \\
            &= \underset{\boldsymbol{x}_{i} \sim X}{\mathbb{E}}\left[1-\frac{2-2f(\boldsymbol{x}_{i})^\mathrm{T} g(\boldsymbol{x}_{i})}{2}\right] \\
            &=1-\frac{1}{2}\underset{\boldsymbol{x}_{i} \sim X}{\mathbb{E}}\left[\|f(\boldsymbol{x}_{i})-g(\boldsymbol{x}_{i})\|_{2}^{2}\right] \\
            &=1-\frac{1}{2}\mathcal{L}_\text{align}.
        \end{aligned}
    \end{equation*}
\end{proof}

\subsection{Proof of Propositions}
\label{propo_proof}
We now recall Proposition \ref{w1} and \ref{w2}.
\begin{proposition}[Bound of gradient scaling factor w.r.t. $ K $]
	\label{w1}
	Given the anchor feature $ \boldsymbol{f}_i $, and temperature $ \tau $, if $ K \rightarrow +\infty $, then $ \mathcal{W}_i $ approaches its upper bound 1. The limit is formulated as:
	\begin{equation*}
		\lim _{K \rightarrow+\infty}\mathcal{W}_i = 1.
	\end{equation*}
\end{proposition}
\begin{proposition}[Bound of gradient scaling factor w.r.t. $ \tau $]
	\label{w2}
	Given $ \boldsymbol{f}_i $ and key set $ G $, $\mathcal{W}_i$ monotonically changes with respect to $ \tau$. The monotonicity is determined by the similarity distribution of samples. If $ \tau \rightarrow +\infty $, then $ \mathcal{W}_i $ approaches its bound $ K/(K+1) $, formulated as:
	\begin{equation*}
		\lim _{\tau \rightarrow+\infty}\mathcal{W}_i = \frac{K}{1+K}.
	\end{equation*}
\end{proposition}
For simplicity, let $ E_k= \exp \left(\boldsymbol{f}_i ^\mathrm{T} \boldsymbol{g}_{k} / \tau\right)$, $ s_k = \boldsymbol{f}_i ^\mathrm{T}\boldsymbol{g}_{k}$, $ E_{max}=\max(E_1,\cdots, E_k, \cdots, E_K), k \neq i $, and $ E_{min}=\min(E_1,\cdots, E_k, \cdots, E_K), k \neq i $.

\begin{proof}[Proof of Proposition \ref{w1}.]
	\label{prf1}
	Here \begin{equation}
		\mathcal{W}_i = 1- \frac{E_i}{E_i+ \sum_{j=1}^{K} E_j}, \end{equation} and the following inequality
	\begin{equation} 
		1- \frac{E_i}{E_i+K\cdot E_{min}} \leq \mathcal{W}_i  \leq  1- \frac{E_i}{E_i+K\cdot E_{max}}.  \end{equation} 
	Since we have the limit of the left part 
	\begin{equation*}
		\begin{aligned}
			&\lim _{K \rightarrow+\infty} (1- \frac{E_i}{E_i+K\cdot E_{min}}) \\
			= 	&\lim _{K \rightarrow+\infty} (1- \frac{E_i / K}{E_i / K+ E_{min}}) \\
			= &1,
		\end{aligned}
	\end{equation*}
	as well as the one of the right part
	\begin{equation*}
		\lim _{K \rightarrow+\infty} (1- \frac{E_i}{E_i+K\cdot E_{max}}) =1, 
	\end{equation*}
	thus the limit of $ \mathcal{W}_i $ is
	\begin{equation*}
		\lim _{K \rightarrow+\infty} \mathcal{W}_i =1.
	\end{equation*}
	Notice that $ E_k > 0 $ strictly, then for a given $ K $, $ \mathcal{W}_i < 1 $. Thus, $ \mathcal{W}_i $ has its upper bound of 1 w.r.t. $ K $.
\end{proof}

\begin{proof}[Proof of  Proposition \ref{w2}.]
	\label{prf2}
	For the temperature $ \tau $, we have 
	\begin{equation}
		\label{lim_tau1}
		\begin{aligned}
			\lim _{\tau \rightarrow+\infty} \mathcal{W}_i &= 	\frac{\lim _{\tau \rightarrow+\infty} \sum_{r=1}^{K} E_r}{\lim _{\tau \rightarrow+\infty} E_i+ \lim _{\tau \rightarrow+\infty} \sum_{j=1}^{K} E_j} \\
			&= \frac{\sum_{r=1}^{K} \lim _{\tau \rightarrow+\infty} E_r}{\lim _{\tau \rightarrow+\infty} E_i+ \sum_{j=1}^{K} \lim _{\tau \rightarrow+\infty} E_j}.
		\end{aligned}
	\end{equation}
	Since the similarity value on hypersphere is bounded, i.e., $ s_k = \boldsymbol{f}_i \cdot g_{k} \in [-1,1]$, so
	\begin{equation}
		\label{lim_E}
		\lim _{\tau \rightarrow+\infty} E_k =1.
	\end{equation}
	Hence, from Eqn. (\ref{lim_tau1}) and (\ref{lim_E})
	\begin{equation*}
		\lim _{\tau \rightarrow+\infty} \mathcal{W}_i =\frac{K}{1+K}.
	\end{equation*}
	
	The gradient of $ \mathcal{W}_i $ with respect to $ \tau $ is derived as:
	\begin{equation}
		\frac{\partial \mathcal{W}_i}{\partial \tau} = \frac{1}{{\tau}^2} \cdot \frac{E_i}{(E_i+ \sum_{r=1}^{K} E_j)^2} \cdot \sum_{j=1}^{K} (s_i-s_r) \cdot E_r.
	\end{equation}
	As $ E_k > 0 $, then we have
	\begin{equation}
		\frac{\partial \mathcal{W}_i}{\partial \tau} \propto  \sum_{j=1}^{K} (s_i-s_r) \cdot E_r.
	\end{equation}
	For a batch of very poor embeddings, $ {\partial \mathcal{W}_i}/{\partial \tau} \leq 0$, then $ \mathcal{W}_i $ is a monotonic decreasing function with respect to $ \tau $. In contrast, for a batch of good embeddings, $ \mathcal{W}_i $ monotonically increases as $ \tau $ increases. So the similarity distribution of samples determine the monotonicity. 
	
	Naturally, Proposition \ref{w bound2} is a direct consequence of above conclusions.
\end{proof}

\section{Implementation Details and Further Discussions}
\subsection{Experiments on ImageNet-1K}
\label{Setup Details}
\textit{For MACL implementation on SimCLR framework}, we follow the original augmentations (random crop, resize, random ﬂip, color distortions, and Gaussian blur). The projection head is a 2-layer MLP projecting the representation to a 128-dimensional latent space. Models optimizations are completed by LARS with a base learning rate of 0.3 (0.3$\times$BatchSize/256) and weight decay of 1e-6. We also use the cosine decay learning rate schedule with 10 epochs warmup. Parameters $\{\tau_0, \alpha, \mathcal{A}_0\}$ are set to $\{0.1, 0.5, 0\}$. \textit{For MACL implementation on MoCo V2 framework}, we experiment on ImageNet-100, and the settings are listed in Appendix \ref{queue_size_exp}. Codes of models are implemented on mmselfsup \cite{mmselfsup2021} with several Tesla A100 80G GPUs.

\subsubsection{Loss Function Ablation}
\label{abla2}
\citet{chen2020simple} ﬁnd the square root learning rate scaling is more desirable with LARS optimizer, i.e., $\text { LearningRate }=0.075 \times \sqrt{\text { BatchSize }}$. Actually, for smaller batch sizes, such a scaling schedule provides a larger learning rate over the linear one, i.e., $\text{LearningRate} =0.3 \times \text{BatchSize} / 256$. For instance, the learning rate of 256-batch size is 1.2 under the square schedule while  0.3 under the linear schedule. Regarding ablations for MACL, we experiment with 512-batch size using SimCLR framework and linear learning rate scaling. We also present the much larger learning rate ablation results in Table \ref{loss_lr_abla}, in which we set it to 2.4. There are some observations. First, similar to the baseline, variants of our MACL achieve significantly better performance under a larger learning rate. LR-l provides an even higher gain than that on the baseline. Besides, the ablations under LR-l also suggest the contributions made by different parts of the proposed loss function. Furthermore, trained for 200 epochs with 512-batch size, only using adaptive temperature or reweighting, our strategy can obtain better accuracies compared to the 512-batch size, 800-epoch or 1024-batch size, 400-epoch of the baseline.

\subsection{Experiments on ImageNet-100}
ImageNet-100 is a subset of ImageNet-1K, in which the images belong to 100 classes. The adopted encoders are ResNet-50 \cite{he2016deep}.
\subsubsection{Queue Size Experiment}
\label{queue_size_exp}
For MoCo v2 \cite{chen2020improved}, we follow their settings (including augmentations and architecture) on ImageNet-1K except for the learning rate of pre-training is 0.3 and a 10 epochs warmup is added. In linear evaluation, we use the batch size of 256 and an SGD optimizer with a learning rate of 10, and momentum of 0.9 without weight decay regularization. Epochs for pre-training and evaluation is 200 and 100, respectively. We set \{$\tau_0$, $\alpha$, $\mathcal{A}_0$\} to \{0.15, 0.5, 0.2\} for MACL experiments and the temperature is 0.2 for original MoCo v2 following their ImageNet-1K setup. The queue size experiment mentioned in Sec. \ref{neg_size} is reported in Table \ref{in100_queue}. Instead of sampling negative samples within a mini-batch, MoCo family exploits a queue structure to store instance representations. From these results, we can see that MoCov2 has better stability in terms of negative size compared to SimCLR. Actually, \textit{MoCo is less likely to be troubled with easy positive pairs} since the utilized momentum encoder is updated slowly (momentum value is 0.999). And the synchronous update framework with weight-shared networks such as SimCLR is more likely to encode the same instance similarly, then is more sensitive to the gradient reduction dilemma. Even so, models have better performance with MACL strategy.
\begin{table}[htb]
  \centering
  \caption{\textbf{Effect of queue sizes} (top-1/top-5 linear evaluation accuracies on ImageNet-100 with 200-epoch pre-training).}
    \begin{tabular}{ccccc}
    \toprule[1.2pt]
    Queue size & 256  & 512  & 4096 & 65536 \\
    \midrule
    MoCo v2                  & 76.80 / 94.34     & 76.89 / 94.24     & 77.02 / 94.31     & 76.36 / 93.92  \\
    \multirow{2}{*}{\textbf{w/ MACL}} & \textbf{77.10} / \textbf{94.36}     & \textbf{77.24} / \textbf{94.39}      & \textbf{77.62} / \textbf{94.45}     & \textbf{77.46} / \textbf{94.16}  \\
    & (\textcolor[RGB]{116,0,0}{+0.30}) / (\textcolor[RGB]{116,0,0}{+0.02})  & (\textcolor[RGB]{116,0,0}{+0.35}) / (\textcolor[RGB]{116,0,0}{+0.15}) & (\textcolor[RGB]{116,0,0}{+0.60}) / (\textcolor[RGB]{116,0,0}{+0.14}) & (\textcolor[RGB]{116,0,0}{+1.10}) / (\textcolor[RGB]{116,0,0}{+0.24}) \\
    \bottomrule[0.9pt]
    \end{tabular}
  \label{in100_queue}
\end{table}
\subsubsection{Parameter and Ablation Analysis}
\label{simclr_in100}
Regarding the scores listed in Table \ref{parameters}, their settings are the same as that on ImageNet-1K. Similar to the trend of the vanilla NT-Xent loss in \cite{chen2020simple}, too large or small temperatures will lead to improper scaling for positive and negative similarities in Softmax, then plagues the CL. Thus, searching a proper $\tau_0$ is necessary for the dynamic adaptation and we refer to the value of the fixed ones of original methods for our settings, e.g., 0.1 for SimCLR \cite{chen2020simple}. $\alpha$ can determine the change range of temperature, and we find that 0.5 provides a higher gain within this group of alternatives. $\mathcal{A}_0$ is the initial alignment threshold related to the change direction of $\tau_a$. Too large $\mathcal{A}_0$ will lead to extremely small temperature in the early training period as alignment magnitude $\mathcal{A}$ is low. Overall, the final temperature in MACL is adaptive to alignment magnitude and scaled by these three factors. Since $\tau_0$ is the datum point, models are more sensitive to its setting. Choosing appropriate parameters enable CL models to deal with uniformity-tolerance dilemma better.

We further conduct comparisons with NNCLR \cite{dwibedi2021little} on ImageNet-100 and present them in Table \ref{nnclr_abla}. It is worth noting that the InfoNCE objective construction in NNCLR is different from that in SimCLR and MoCo family. NNCLR obtains the positive key from a support set using nearest-neighbours to increase the richness of latent representation and go beyond single instance positives. As such, the positive pair of representations might belong to distinct instances. We set $\tau_0$ and $\tau$ to 0.1 and use LARS optimizer following NNCLR literature, learning rate is 0.8, and cosine decay schedule with 10 epochs warmup. We find that under different parameters our MACL can generally outperform the original version and has the biggest 2.22 / 1.28 percent gain of top-1/top-5 accuracy. The performance demonstrates that our MACL is also applicable for such a support set framework to facilitate contrastive learning.
\begin{table*}[htb]
\setlength{\tabcolsep}{3pt}
  \centering
  \caption{\textbf{Ablation comparisons on ImageNet-100 with  NNCLR framework} (top-1/top-5 linear evaluation accuracies with 100-epoch pre-training, temperature 0.1, and 512-batch size).}
    \begin{tabular}{cccccc|c}
    \toprule[1.2pt]
    $\alpha$ & \multicolumn{3}{c}{0.5} & \multicolumn{2}{c|}{1} & \multirow{2}[2]{*}{NNCLR} \\
    \cmidrule(lr){2-4} \cmidrule(lr){5-6}
    $\mathcal{A}_0$    & 0     & 0.2   & 0.6   & 0     & 0.6   &  \\
    \hline
    Acc. & 67.12 / 89.92 & \textbf{67.72 / 90.02} & 66.76 / 89.45 & 65.90 / 88.99 & 66.56 / 89.16 & 65.50 / 88.74 \\
    \bottomrule[0.9pt]
    \end{tabular}
  \label{nnclr_abla}
\end{table*}

\subsection{Experiments on CIFAR10}
\label{cifar_exp}
Encoders are CIFAR version ResNet-18 \cite{he2016deep}, in which the kernel size of the first 7$ \times $7 convolution is replaced with a 3$ \times $3 one, and the first max pooling module is removed. Unless otherwise stated, the temperature is 0.1 for all the losses and $\alpha$=0.5, $\mathcal{A}_0$=0 for MACL. We make the loss symmetric in implementation and use four types of augmentations for pretraining: random cropping and resizing, random color jittering, random horizontal ﬂip, and random grayscale conversion. The LARS optimizer in SimCLR \cite{chen2020simple} is replaced by Adam with a base learning rate of 1e-3 and weight decay is 1e-6. For batch sizes that are larger than 256, the learning is scaled by 1e-3$\times$Batchsize/256. We train the encoders for 200 epochs. For linear evaluation, the trained CL models are evaluated by fine-tuning a linear classiﬁer for 100 epochs with 128-batch size on top of frozen backbones. We utilize an SGD optimizer by setting the learning rate to 0.02, momentum to 0.9, and weight decay to 0.

\begin{table}[htb]
\setlength{\tabcolsep}{2.5pt}
	\centering
	\caption{\textbf{Comparison of reweighting methods} (top-1 linear evaluation / kNN accuracies on CIFAR10, $k=200$).}
	\begin{tabular}{cccccc}
		\toprule[1.2pt]
		\multicolumn{1}{c}{Batch size} & 64    & 128   & 256   & 512   & 1024 \\
		\midrule
		NT-Xent & 82.31 / 78.80  & 83.56 / 79.78  & 84.65 / 81.46  & 85.13 / 81.91   & 85.30 / 82.27 \\
        \midrule
        FlatNCE & 86.30 / 84.50  & 86.28 / 84.47  & 86.11 / 84.08  & 86.02 / 83.99   & 85.84 / 83.54 \\
		DCL     & 86.28 / 84.59  & 86.04 / 84.64  & 86.29 / 83.86  & 86.33 / 84.02   & 85.61 / 83.07 \\
		Dual    & 86.32 / 84.40  & 86.40 / 84.69  & 85.86 / 83.87  & 86.23 / 83.75   & 86.05 / 83.64 \\
        \midrule
		\textbf{MACL}  & \textbf{87.11} / \textbf{84.96} & \textbf{87.41} / \textbf{84.85} & \textbf{87.27} / \textbf{85.32}  & \textbf{87.24} / \textbf{85.18} & \textbf{86.75} / \textbf{84.71} \\
		\bottomrule[0.9pt]
	\end{tabular}
	\label{cifar10_batch}
\end{table}

\subsection{Sentence Embedding Experiments}
\label{simcse_setting}
Pre-training is completed on the 1-million randomly sampled sentences from English Wikipedia, which is the same as SimCSE. Following \cite{gao2021simcse}, learning starts from pre-trained checkpoints of the base version RoBERTa(cased) \cite{liu2019roberta} and BERT(uncased) \cite{kenton2019bert}. We set $\{\tau_0, \alpha, \mathcal{A}_0\}$ to $\{0.05, 2, 0.8\}$. Following \cite{gao2021simcse}, algorithms are performed based on Huggingface’s transformers package\footnote{https://github.com/huggingface/transformers,version 4.2.1.} and evaluated with SentEval toolkit\footnote{https://github.com/facebookresearch/SentEval}. The exploited Wikipedia sentence dataset is the one released by SimCSE authors. Models are trained for 1 epoch. For SimCSE, only dropout is exploited as augmentation, models have a good initial alignment for positive pairs \cite{gao2021simcse}. The batch size is set as 64, and learning rate for BERT version is 3e-5 and 1e-5 for the RoBERTa one. We try a stronger dropout in the experiment and found that the rate of 0.2 can generate better scores when cooperating with MACL, but is not suitable for vanilla InfoNCE. Note that as the original literature shows that the results are not sensitive to batch size, so we did not apply reweighting in this part.
\begin{table*}[htb]
  \centering
  \caption{\textbf{STS tasks comparisons of sentence embeddings} (the adopted metric is Spearman’s correlation with “all” setting).}
    \begin{tabular}{ccccccccc}
    \toprule[1.2pt]
    STS task & STS12 & STS13 & STS14 & STS15 & STS16 & STSB  & SICKR & Avg. \\
    \midrule
    SimCSE-BERT & \textbf{68.40 } & 82.41  & 74.38  & 80.91  & 78.56  & 76.85  & 72.23  & 76.25  \\
    \multirow{2}{*}{\textbf{w/ MACL}} & 67.16  & \textbf{82.78 } & \textbf{74.41 } & \textbf{82.52 } & \textbf{79.07 } & \textbf{77.69 } & \textbf{73.00 } & \textbf{76.66 } \\
    & (\textcolor[RGB]{34,139,34}{-1.24}) & (\textcolor[RGB]{116,0,0}{+0.36}) & (\textcolor[RGB]{116,0,0}{+0.03}) & (\textcolor[RGB]{116,0,0}{+1.61}) & (\textcolor[RGB]{116,0,0}{+0.51}) & (\textcolor[RGB]{116,0,0}{+0.84}) & (\textcolor[RGB]{116,0,0}{+0.77}) & (\textcolor[RGB]{116,0,0}{+0.41})\\
    \midrule
    \midrule
    Transfer task & MR    & CR    & SUBJ  & MPQA  & SST2  & TREC  & MRPC  & Avg. \\
    \midrule
    SimCSE-BERT & 81.18  & \textbf{86.46 } & 94.45  & 88.88  & 85.50  & \textbf{89.80 } & 74.43  & 85.81  \\
    \multirow{2}{*}{\textbf{w/ MACL}}     & \textbf{81.80 } & 86.12  & \textbf{94.66 } & \textbf{89.12 } & \textbf{86.38 } & 88.60  & \textbf{76.46 } & \textbf{86.16 } \\
     & (\textcolor[RGB]{116,0,0}{+0.62}) & (\textcolor[RGB]{34,139,34}{-0.34}) & (\textcolor[RGB]{116,0,0}{+0.22}) & (\textcolor[RGB]{116,0,0}{+0.24}) & (\textcolor[RGB]{116,0,0}{+0.88}) & (\textcolor[RGB]{34,139,34}{-1.20}) & (\textcolor[RGB]{116,0,0}{+2.03}) & (\textcolor[RGB]{116,0,0}{+0.34})\\
    \bottomrule[0.9pt]
    \end{tabular}
  \label{bert_sts_trans}%
\end{table*}

Same as the authors reminded, we also notice that the results are slightly different when implemented on different machines and CUDA versions (all package versions are the same as the author provided). But our MACL indeed can boost the performance on different machines. We try to experiment on Nvidia RTX 3090 with CUDA11.6, RTX 1080ti with CUDA11.4, and Tesla T4 with CUDA11.2 on Google colab\footnote{https://colab.research.google.com} and finally report the results on Tesla T4. In fact, if compared against the reproduced results, our approach has an even more significant improvement. For example, the comparison on Tesla T4 is shown in Table \ref{bertsts_repro}. We can see that the average score on STS tasks has a 1.57 and 0.89 improvement when using MACL strategy with RoBERTa and BERT, respectively.
\begin{table*}[htb]
  \centering
  \caption{\textbf{Reproduction} of sentence embedding performance on STS tasks.}
    \begin{tabular}{ccccccccc}
    \toprule[1.2pt]
    STS task & STS12 & STS13 & STS14 & STS15 & STS16 & STSB  & SICKR & Avg. \\
    \midrule
    SimCSE-RoBERTa & 70.16  & \textbf{81.77 } & 73.24  & 81.36  & 80.65  & 80.22  & 68.56  & 76.57  \\
    SimCSE-RoBERTa (repro) & 67.88  & 81.55  & 72.44  & 81.31  & 80.73  & 80.38  & 67.83  & 76.02  \\
    {\textbf{w/ MACL}} & \textbf{70.76 } & 81.43  & \textbf{74.29 } & \textbf{82.92 } & \textbf{81.86 } & \textbf{81.17 } & \textbf{70.70 } & \textbf{77.59 } \\
    \midrule
    \midrule
    SimCSE-BERT & \textbf{68.40 } & 82.41  & 74.38  & 80.91  & 78.56  & 76.85  & 72.23  & 76.25  \\
    SimCSE-BERT (repro) & {68.26 } & 81.60  & 72.98  & 81.47  & 77.91  & 76.90  & 71.30  & 75.77  \\
    {\textbf{w/ MACL}} & 67.16  & \textbf{82.78 } & \textbf{74.41 } & \textbf{82.52 } & \textbf{79.07 } & \textbf{77.69 } & \textbf{73.00 } & \textbf{76.66 } \\
    \bottomrule[0.9pt]
    \end{tabular}
  \label{bertsts_repro}%
\end{table*}

\subsection{Graph Representation Experiments}
\label{graph_trans}
All of the augmentations and hyper-parameters except for those about loss function are taken from the baseline directly \cite{you2020graph}. $\tau_0$ is set to 0.2 in unsupervised classification and 0.1 in transfer learning. $\{\alpha,\mathcal{A}_0\}$ are set to $\{0.5, 0\}$. The contrastive loss utilized in GraphCL \cite{you2020graph} actually is DCL \cite{yeh2022decoupled}, in which the positive similarity is removed from the denomination of InfoNCE. The transfer learning section is molecular property prediction in chemistry following \cite{you2020graph}. The adopted GNN-based encoders are from \cite{hu2020strategies}. Experiments are performed ten times and finally report the mean and standard deviation of ROC-AUC scores (\%). From Table \ref{graph_transfer}, we can see that MACL has the largest 2.97 percent improvement on MUV dataset and outperforms GraphCL on 6/8 dataset.

\end{document}